\documentclass{article}
\usepackage{ijcai26}

\usepackage{times}
\usepackage{soul}
\usepackage{url}
\usepackage[hidelinks]{hyperref}
\usepackage[utf8]{inputenc}
\usepackage[small]{caption}
\usepackage{graphicx}
\usepackage{amsmath}
\usepackage{amsthm}
\usepackage{booktabs}
\usepackage{algorithm}
\usepackage{algorithmic}
\usepackage[switch]{lineno}

\usepackage{amsmath,amssymb}
\usepackage{mathtools}
\usepackage{bbm}
\usepackage{graphicx}
\usepackage{enumitem}
\usepackage{hyperref}
\usepackage{microtype}

\usepackage{caption}
\usepackage{subcaption}
\usepackage{graphicx}

\newcommand{\A}{\mathcal{A}}
\newcommand{\Ai}{\mathcal{A}_i}

\newcommand{\Ami}{\mathcal{A}_{-i}}

\newcommand{\R}{\mathbb{R}}
\newcommand{\E}{\mathbb{E}}

\newcommand{\KL}{\mathrm{KL}}

\newcommand{\softmax}{\mathrm{Softmax}}

\newtheorem{assumption}{Assumption}
\newtheorem{remark}{Remark}

\newtheorem{corollary}{Corollary}

\newtheorem{theorem}{Theorem}

\title{DIML: \underline{D}ifferentiable \underline{I}nverse \underline{M}echanism \underline{L}earning from Behaviors of Multi-Agent Learning Trajectories}

\author{
    Zhiyu An
    \qquad
    Wan Du
    \affiliations
    University of California, Merced
    \emails
    \{zan7, wdu3\}@ucmerced.edu
}

\begin{document}

\maketitle

\begin{abstract}
We study \emph{inverse mechanism learning}: recovering an unknown incentive-generating mechanism from observed strategic interaction traces of self-interested learning agents.
Unlike inverse game theory and multi-agent inverse reinforcement learning, which typically infer utility/reward \textit{parameters} inside a structured mechanism, our target includes unstructured mechanism---a (possibly neural) mapping from joint actions to per-agent payoffs.
Unlike differentiable mechanism design, which optimizes mechanisms forward, we infer mechanisms from behavior in an observational setting.
We propose \textbf{DIML}, a likelihood-based framework that differentiates through a model of multi-agent learning dynamics and uses the candidate mechanism to generate counterfactual payoffs needed to predict observed actions.
We establish identifiability of payoff differences under a conditional logit response model and prove statistical consistency of maximum likelihood estimation under standard regularity conditions.
We evaluate DIML with simulated interactions of learning agents across unstructured neural mechanisms, congestion tolling, public goods subsidies, and large-scale anonymous games.
DIML reliably recovers identifiable incentive differences and supports counterfactual prediction, where its performance rivals tabular enumeration oracle in small environments and its convergence scales to large, hundred-participant environments.
Code to reproduce our experiments is open-sourced.\footnote{https://github.com/ryeii/DIML-Differentiable-Inverse-Mechanism-Learning}

\end{abstract}

\section{Introduction}
Incentive mechanisms in modern socio-technical systems govern behavior in many of today’s most influential multi-agent systems, and they are increasingly implemented by opaque payout rules.
For example, advertisement platforms allocate visibility and payments through proprietary ranking and rebate rules; transportation and energy systems impose congestion charges and dynamic tolls; and digital labor and content platforms deploy bonuses, penalties, and revenue-sharing schemes to steer participant behavior.
In all of these settings, the mechanism (a rule mapping agents’ joint actions to payoffs) is often proprietary, adaptive, or learned, while the observable data consists primarily of agents’ interaction traces over time.
Understanding, auditing, or predicting the effects of such systems therefore requires the ability to reason about incentives without direct access to the mechanism itself.
This motivates a fundamental inverse problem: 
\begin{center}
    \emph{Can we infer the incentive-generating mechanism from observed strategic behavior alone?}
\end{center}
Solving this problem enables high-impact applications in auditing and regulation of incentive-driven systems, such as reconstructing hidden pricing, subsidy, or transfer rules from interaction data, supporting fairness audits, market oversight, and diagnosis of strategic dynamics in platforms and auctions.
It also enables counterfactual analysis, i.e. predicting how behavior would change under alternative populations or policies when the mechanism itself is unobserved.

Despite the centrality of incentives in these domains, existing methods fall short of addressing this inverse problem in realistic settings.
Several adjacent literatures study inverse problems in strategic or multi-agent environments, but they target different unknowns, operate at mismatched causal levels, or rely on assumptions that limit their applicability.
For example,
classical \emph{inverse game theory} seeks to infer payoff parameters that rationalize observed play under equilibrium, regret, or correlated equilibrium notions~\cite{waugh2013computational,bestick2013inverse}.
Recent extensions improve computational efficiency and robustness~\cite{goktas2025efficient,cui2025inverse}, but the inferred object remains agents’ utilities within a structured mechanism, rendering these method powerless when the mechanism is unstructured (e.g. neural).
In addition, these approaches typically assume equilibrium or stationary behavior and do not exploit the rich transient dynamics produced by learning agents.
As a result, they inherit fundamental non-identifiability issues: many payoff structures can rationalize the same equilibrium behavior.

A second line of work studies \emph{multi-agent inverse reinforcement learning} (MA-IRL), which infers agents’ reward functions in Markov games under specified solution concepts~\cite{yu2019multi}.
Recent work emphasizes reward ambiguity and feasible reward sets~\cite{freihaut2024feasible} or incorporates bounded rationality and theory-of-mind reasoning~\cite{wu2023multiagent}.
While powerful, MA-IRL suffers the same flaw as previous works, treating the mechanism as structured and attributes behavior to latent agent preferences.
In contrast, in many platform and market settings, incentives are externally imposed and intentionally coupled across agents through a mechanism.
Inferring independent rewards ignores this incentive coupling and conflates agent preferences with mechanism-induced payoffs.

A third, complementary literature focuses on \emph{differentiable mechanism design} and incentive optimization.
Neural mechanism design frameworks such as RegretNet~\cite{dutting2017optimal} and AMenuNet~\cite{duan2023scalable} parameterize mechanisms as neural networks and optimize them forward for revenue or welfare.
Related work in bilevel and meta-gradient optimization differentiates through agents’ learning dynamics to design incentives that shape behavior~\cite{yang2021adaptive,thoma2024contextual}.
These approaches demonstrate that incentive rules can be learned and differentiated through, but they assume full control and observability of the mechanism.
They do not address the inverse problem of reconstructing an unknown mechanism from observed behavior.

Taken together, existing work either infers \emph{preferences rather than incentives}, assumes \emph{equilibrium rather than learning}, or solves a \emph{forward design problem rather than an inverse reconstruction problem}.
To our knowledge, no prior work directly studies how to infer a high-capacity incentive mechanism from multi-agent learning trajectories.

\paragraph{Core challenges.}
Inverse mechanism learning poses several intertwined challenges.
First, payoffs are typically unobserved, and incentive mechanisms may be high-dimensional or neural, ruling out tabular or closed-form inversion.
Second, equilibrium behavior alone is insufficient for identification: distinct mechanisms can induce identical equilibria.
Third, observed behavior is potentially generated by \emph{learning agents}, whose updates depend on counterfactual payoffs for actions not taken.
Any successful inverse method must therefore (i) exploit off-equilibrium behavior, (ii) scale to expressive mechanism classes, and (iii) explicitly model how incentives shape learning dynamics over time.

\paragraph{Our approach.}
We propose \emph{Differentiable Inverse Mechanism Learning} (DIML), a framework that reconstructs an unknown incentive mechanism by differentiating through a model of multi-agent learning.
The key insight is that, given a candidate mechanism $M_\theta$, one can evaluate it counterfactually on joint actions $(a_i',a_{-i})$ to compute the payoffs an agent would have received had it deviated.
These counterfactual payoffs are precisely the quantities required by many learning rules, including logit response, no-regret updates, and policy-gradient-style learners.
By unrolling the learning dynamics and matching the induced action distributions to observed trajectories, DIML defines a differentiable likelihood over trajectories as a function of the mechanism parameters $\theta$.

This perspective shifts the causal level of inference.
Rather than asking which rewards rationalize observed actions, DIML asks which \emph{incentive structure} explains how agents adapt over time.
Rather than relying on equilibrium assumptions, it leverages transient learning dynamics as identifying signal.
Rather than optimizing incentives forward, it enables inverse reconstruction and counterfactual auditing of existing systems.

DIML resolves the challenges of inverse mechanism learning in three ways.
First, by modeling learning explicitly, it exploits off-equilibrium behavior that equilibrium-based inverse methods discard, reducing non-identifiability.
Second, by parameterizing mechanisms as neural networks and optimizing via backpropagation, it scales to complex, high-capacity incentive rules.
Third, by separating behavioral fit from mechanism recovery and counterfactual validity, it enables principled evaluation of whether an inferred mechanism captures true incentives rather than merely imitating observed behavior.

\paragraph{Contributions.}
We make the following contributions:
\begin{itemize}[leftmargin=1.25em]
    \item We formalize \emph{inverse mechanism learning from multi-agent learning behavioral trajectories} as a distinct and practical problem, clarifying its relationship to inverse game theory, MA-IRL, and differentiable mechanism design.
    \item We introduce \textbf{DIML}, a general likelihood-based framework that infers neural incentive mechanisms by differentiating through multi-agent learning dynamics.
    \item We provide identifiability and consistency results for payoff differences under a conditional logit response model, elucidating inherent gauge freedoms and conditions under which inverse mechanism learning is well-posed.
    \item We propose an evaluation protocol that distinguishes behavioral imitation from mechanism recovery and counterfactual validity, and benchmark DIML across a diverse suite of environments and mechanism families.
\end{itemize}

\section{Related Work}

\begin{table*}[t]
\centering
\small
\caption{Comparison with related work.
Legend:
Inf/Beh = inference from behavior;
Mec/Tar = mechanism is inference target (joint-action $\rightarrow$ payoff map);
N/Mec = neural or high-capacity mechanism;
MA = multi-agent setting;
L/Traj = uses learning trajectories (not only equilibrium/static play);
Diff = differentiates through learning or response dynamics;
N/PO = no payoff observation required;
CF = supports counterfactual evaluation.
Direction labels indicate the primary research paradigm of each work.
Horizontal rules group papers by research direction: inverse utility inference, forward mechanism design, recent inverse-theoretic work, and DIML.}
\label{tab:related}
\begin{tabular}{llccccccccc}
\toprule
Direction &
Work &
Inf/Beh &
Mec/Tar &
N/Mec &
MA &
L/Traj &
Diff &
N/PO &
CF \\
\midrule
IGT &
MaxEnt-IGT~\cite{waugh2013computational} &
\checkmark & & & \checkmark & & & \checkmark & \\
IGT &
ICE-IGT~\cite{bestick2013inverse} &
\checkmark & & & \checkmark & & & \checkmark & \\
MAIRL &
MA-AIRL~\cite{yu2019multi} &
\checkmark & & & \checkmark & \checkmark & & \checkmark & \checkmark \\
MAIRL &
ToM-MAIRL~\cite{wu2023multiagent} &
\checkmark & & & \checkmark & \checkmark & & \checkmark & \checkmark \\
\midrule
MechDes &
RegretNet~\cite{dutting2017optimal} &
 & \checkmark & \checkmark & \checkmark & & & & \\
MechDes &
AMenuNet~\cite{duan2023scalable} &
 & \checkmark & \checkmark & \checkmark & & & & \\
Bilevel &
CB-RL~\cite{thoma2024contextual} &
 & \checkmark &  & \checkmark & \checkmark & \checkmark & & \checkmark \\
Bilevel &
PA-Bandits~\cite{scheid2024incentivized} &
 & \checkmark &  &  & \checkmark & \checkmark & & \checkmark \\
\midrule
InvTheory &
Feasible-MAIRL~\cite{freihaut2024feasible} &
\checkmark & & & \checkmark &  &  & \checkmark & \\
InvTheory &
IML~\cite{goktas2025efficient} &
\checkmark & & & \checkmark &  &  & \checkmark & \\
InvTheory &
Incenter-IGT~\cite{cui2025inverse} &
\checkmark & & & \checkmark &  &  & \checkmark & \\
InvTheory &
Entropy-IGT~\cite{liaodecoding} &
\checkmark & & & \checkmark &  &  & \checkmark & \\
\midrule
\textbf{InvMech} &
\textbf{DIML (ours)} &
\checkmark & \checkmark & \checkmark & \checkmark & \checkmark & \checkmark & \checkmark & \checkmark \\
\bottomrule
\end{tabular}
\end{table*}


\subsection{Inverse game theory and inverse equilibrium}
Inverse game theory infers payoff parameters that rationalize observed play under equilibrium, regret, or related constraints.
Representative approaches include inverse equilibrium via regret and maximum entropy principles~\cite{waugh2013computational} and inverse correlated equilibrium formulations~\cite{bestick2013inverse}.
These methods typically infer utilities within a fixed game form and often operate at equilibrium or stationary summary levels rather than learning trajectories.

Recent work continues to emphasize both computational structure and non-identifiability.
Goktas et al.~\cite{goktas2025efficient} develop polynomial-time methods for inverse multiagent learning and related ``simulacral'' variants, formulated as min--max problems under oracle access assumptions.
Cui~\cite{cui2025inverse} proposes an incenter-based approach to inverse game theory via inverse variational inequalities, focusing on robustness to observed equilibria.
These contributions sharpen the inverse-game-theory toolkit, but they still target \emph{player utilities} (or low-dimensional game parameters) rather than reconstructing an \emph{incentive mechanism} as a joint-action-to-payoff map, and typically do not leverage differentiable unrolled learning dynamics as the primary source of signal.

\subsection{Multi-agent IRL, reward ambiguity, and solution concepts}
Multi-agent IRL (MAIRL/MIRL) aims to infer agents' reward functions in Markov games under specific solution concepts.
Adversarial methods such as MA-AIRL~\cite{yu2019multi} scale to high-dimensional settings via (pseudo-)likelihood objectives, while other work models bounded rationality or richer social reasoning.
For instance, Wu et al.~\cite{wu2023multiagent} incorporate Theory-of-Mind reasoning to relax the assumption that agents know each other’s goals a priori.

A central theme is \emph{reward ambiguity}: many reward functions can rationalize the same equilibrium behavior.
Freihaut and Ramponi~\cite{freihaut2024feasible} characterize feasible reward sets in Markov games and use entropy regularization to obtain uniqueness properties and sample complexity analyses.
These results underscore why inferring \emph{agent rewards} is ill-posed without additional structure.
DIML differs in target and in coupling: the unknown is not each agent's private reward in a fixed game, but rather a \emph{shared mechanism} that jointly maps actions (and optionally context) to payoffs, creating strategic coupling across agents.

\subsection{Differentiable mechanism design and bilevel incentive optimization}
Differentiable mechanism design learns mechanisms forward by optimizing welfare while encouraging approximate incentive compatibility, e.g.\ RegretNet~\cite{dutting2017optimal}.
More recent work improves expressivity and incentive properties; for example, Duan et al.~\cite{duan2023scalable} propose a scalable neural architecture for DSIC affine maximizer auctions.
At the theory/engineering boundary, D\"utting et al.~\cite{dutting2017optimal} provide a detailed differentiable-economics treatment of optimal auctions through deep learning.

A separate but closely related line studies incentive design as a bilevel or Stackelberg problem where a principal shapes followers' learning.
Thoma et al.~\cite{thoma2024contextual} introduce Contextual Bilevel Reinforcement Learning (CB-RL) and provide convergence analysis with hypergradients estimated from follower trajectories.
Scheid et al.~\cite{scheid2024incentivized} study principal-agent bandit games where a principal learns an incentive policy to influence an agent with misaligned objectives.
These works are fundamentally \emph{forward} problems: optimize incentives to achieve a designer objective.
DIML instead addresses the \emph{inverse} observational problem: reconstruct an unknown mechanism from action traces, enabling auditing, reverse-engineering, and counterfactual analysis when the mechanism is not directly available.

Across all of the above areas, the recurring gaps are:
(i) equilibrium-centric inference that discards learning dynamics,
(ii) reward inference that treats strategic coupling as fixed rather than mechanism-induced,
and (iii) forward mechanism optimization without inverse reconstruction.
DIML targets the missing regime: \emph{inverse mechanism learning with neural mechanisms from multi-agent learning trajectories}.

\section{Problem Setup}
We consider $n$ agents indexed by $i\in\{1,\dots,n\}$ with finite action sets $\Ai$, joint action space $\A \coloneqq \prod_{i=1}^n \Ai$, and joint actions $a_{1:n}=(a_1,\dots,a_n)\in\A$.
An \emph{incentive mechanism} is a mapping
\[
M_\theta:\A\to \R^n,\qquad M_\theta(a_{1:n}) = (r_1,\dots,r_n),
\]
parameterized by $\theta$ (e.g., a neural network).

We observe a dataset of interaction trajectories
\[
D=\{\tau^{(m)}\}_{m=1}^M,\qquad \tau^{(m)} = (a_{1:n}^{(m)}(t))_{t=1}^{T_m},
\]
generated by agents learning while interacting under an unknown ground-truth mechanism $M_{\theta^\star}$.
The primary setting assumes only actions are observed; optional variants incorporate observed payoffs.

\subsection{Agent learning model and induced trajectory likelihood}
DIML assumes a differentiable family of learning dynamics $\Pi$.
In the main method and experiments we use a differentiable logit-Q model:
each agent maintains scores $Q_i(t)\in\R^{|\Ai|}$; given opponent action $a_{-i}(t)$ and candidate mechanism $M_\theta$, we define counterfactual payoffs
\begin{equation}
\label{eq:counterfactual}
u^\theta_i(a_i, a_{-i}(t)) \;\coloneqq\; [M_\theta(a_i, a_{-i}(t))]_i,\qquad \forall a_i\in\Ai.
\end{equation}
Scores update by an exponentially-weighted moving average
\begin{equation}
\label{eq:q_update}
Q_i(t+1) = (1-\alpha)Q_i(t) + \alpha \, u^\theta_i(\cdot, a_{-i}(t)),
\end{equation}
and actions are chosen via a logit policy with temperature $\beta>0$ and optional exploration $\varepsilon$:
\begin{equation}
\label{eq:logit_policy}
\pi^\theta_i(t+1) \;=\; (1-\varepsilon)\,\softmax(\beta Q_i(t+1)) + \varepsilon\,\mathrm{Unif}(\Ai).
\end{equation}
Assuming conditional independence of agents' actions given internal states $(Q_i(t))_i$, this defines a differentiable likelihood $p_\theta(\tau)$ over trajectories.

\subsection{Inverse mechanism learning objective}
We estimate $\theta$ by maximum likelihood with mechanism-structure regularization: for $\hat\theta \in \arg\min_{\theta}\; \mathcal{L}(\theta;D) + \lambda \mathcal{R}(\theta),$
\begin{equation}
    \label{eq:diml_obj}
\mathcal{L}(\theta;D)\coloneqq -\sum_{m=1}^M \log p_\theta(\tau^{(m)}),
\end{equation}
where $\mathcal{R}(\theta)$ can encode mechanism priors such as bounded payments, smoothness, or approximate budget balance $\E[(\sum_i r_i)^2]$.

\section{DIML: Differentiable Inverse Mechanism Learning}
DIML follows the principle: \emph{use a candidate mechanism to impute counterfactual payoffs needed to predict learning-driven behavior, then fit the mechanism by differentiable likelihood.}
The key modeling move is Equation~\eqref{eq:counterfactual}: since the mechanism is a joint-action-to-payoff map, conditioning on observed opponent actions $a_{-i}(t)$ enables counterfactual evaluation of the payoff the agent would have obtained under any alternative action $a_i'\in\Ai$.
This transforms an action-only dataset into a supervised signal for the mechanism through the agents' learning dynamics.

\subsection{From mechanisms to counterfactual payoff tensors}
For each time step $t$ and each agent $i$, define the \emph{counterfactual payoff vector}
\[
u_i^\theta(\cdot, a_{-i}(t)) \in \R^{|\Ai|},
u_i^\theta(a_i', a_{-i}(t)) \coloneqq [M_\theta(a_i', a_{-i}(t))]_i.
\]
In implementation, DIML constructs a \emph{counterfactual payoff tensor}
\[
U^\theta(t) \in \R^{n \times |\Ai|},\qquad
U^\theta_{i,a}(t) = u_i^\theta(a, a_{-i}(t)),
\]
by enumerating each agent's action while holding opponents fixed at their realized actions.
This tensor is the minimal object needed to update the agents' internal learning state and to evaluate the probability assigned to the next observed action.
Importantly, $U^\theta(t)$ is differentiable in $\theta$ whenever $M_\theta$ is differentiable.

\paragraph{Why counterfactual evaluation is essential.}
If one only evaluates $M_\theta$ on realized joint actions $a_{1:n}(t)$, then $M_\theta$ affects the likelihood only through realized payoffs, which are unobserved in our primary setting.
Counterfactual evaluation provides a mechanism-dependent proxy for the latent feedback signal that learners use to update policies, making inverse recovery possible from actions alone.

\subsection{Learning dynamics model}
DIML requires a differentiable model of how agents map past interaction to future actions.
We instantiate this with logit-Q dynamics (Eq.~\eqref{eq:q_update}--\eqref{eq:logit_policy}) for three reasons:
(i) it yields a tractable likelihood,
(ii) it captures bounded rationality and exploration through $(\beta,\varepsilon)$,
and (iii) it is a common surrogate for a broad class of smooth best-response and stochastic approximation learning rules.
Other differentiable dynamics (e.g., multiplicative weights, policy-gradient updates) can be substituted provided they consume counterfactual payoffs and produce differentiable choice probabilities.

\subsection{Unrolled likelihood and differentiation through learning}
Given a trajectory $\tau=(a_{1:n}(t))_{t=1}^T$ and initial scores $Q(1)$, the log-likelihood under~\eqref{eq:q_update}--\eqref{eq:logit_policy} factorizes as
\begin{equation}
\label{eq:traj_ll_factor}
\log p_\theta(\tau) \;=\; \sum_{t=1}^{T-1}\sum_{i=1}^n \log \pi_i^\theta(t+1)\bigl(a_i(t+1)\bigr),
\end{equation}
where $\pi_i^\theta(t+1)=\pi_i(Q_i(t+1))$ and
\[
Q_i(t+1)=\mathcal{U}\bigl(Q_i(t),\,u_i^\theta(\cdot,a_{-i}(t))\bigr)
\]
for the update operator $\mathcal{U}$ in Eq.~\eqref{eq:q_update}.
Thus $\theta$ affects the likelihood only through the sequence of counterfactual payoff tensors $\{U^\theta(t)\}_{t=1}^{T-1}$ and the induced internal states $\{Q(t)\}$.

\paragraph{Conditional independence assumption.}
To obtain Eq.~\eqref{eq:traj_ll_factor}, DIML assumes that conditional on internal states $(Q_i(t))_{i=1}^n$, agents choose actions independently:
$
p_\theta(a_{1:n}(t+1)\mid Q(t+1)) = \prod_{i=1}^n \pi_i^\theta(t+1)(a_i(t+1)).
$
This is standard in MARL system identification and allows efficient training.
Notably, this does \emph{not} assume independence in the strategic process; coupling enters through the mechanism-dependent counterfactual payoffs $u_i^\theta(\cdot,a_{-i}(t))$, which depend on others' realized actions.

\paragraph{Backpropagation through learning.}
We optimize~\eqref{eq:diml_obj} by differentiating the unrolled computation graph:
\[
\frac{\partial}{\partial\theta}\log p_\theta(\tau)
= \sum_{t,i}\frac{\partial}{\partial\theta}\log \pi_i^\theta(t+1)(a_i(t+1)),
\]
where each term differentiates through (i) the policy map $\pi_i(\cdot)$, (ii) the learning update recursion for $Q_i$, and (iii) the counterfactual mechanism evaluations forming $u_i^\theta(\cdot,a_{-i}(t))$.
In practice, we apply truncated BPTT for long horizons and optionally checkpoint intermediate states to control memory.

\noindent\textbf{Computational complexity and scaling.} For discrete actions, computing $U^\theta(t)$ by naive enumeration requires $O(n|\Ai|)$ forward passes of $M_\theta$ per time step.
When $M_\theta$ is a neural network evaluated on batches, this is efficiently vectorized: for each agent $i$, one constructs a batch of $|\Ai|$ counterfactual joint actions that differ only in coordinate $i$ and runs a single batched forward pass.
Overall time per trajectory scales as
\[
O\bigl(T \cdot n \cdot |\Ai| \cdot C_{M}\bigr),
\]
where $C_M$ is the cost of a forward pass of $M_\theta$.
This cost is the price of action-only identification: we must evaluate off-path (counterfactual) outcomes.
For larger action spaces, one can replace full enumeration with (i) importance-sampled subsets of counterfactual actions, (ii) structured factorization of $M_\theta$ that reduces evaluation cost, or (iii) continuous-action variants using gradients rather than enumeration.

\subsection{Algorithm}
Algorithm~\ref{alg:diml} summarizes DIML for stateless mechanisms.
The contextual (Markov) variant extends $M_\theta$ to $M_\theta(s_t,a_{1:n}(t))$ and computes counterfactual payoffs by holding $s_t$ fixed.

\begin{algorithm}[t]
\small
\caption{DIML (unrolled likelihood, stateless)}
\label{alg:diml}
\begin{algorithmic}[1]
\REQUIRE Trajectories $D=\{\tau^{(m)}\}$ of joint actions; learning hyperparameters $(\alpha,\beta,\varepsilon)$; mechanism model $M_\theta$
\STATE Initialize $\theta$ (and optional learnable $Q(1)$)
\FOR{epoch $=1,2,\dots$}
    \FOR{minibatch of trajectories $\tau=(a_{1:n}(t))_{t=1}^T$}
        \STATE Initialize $Q_i(1)$ for all agents $i$
        \STATE $\mathcal{J}\leftarrow 0$
        \FOR{$t=1$ to $T-1$}
            \FOR{agent $i=1$ to $n$}
                \STATE Compute $u_i^\theta(\cdot,a_{-i}(t))$ by evaluating $M_\theta(a_i',a_{-i}(t))$ for all $a_i'\in\Ai$
                \STATE Update $Q_i(t+1)\leftarrow (1-\alpha)Q_i(t) + \alpha\,u_i^\theta(\cdot,a_{-i}(t))$
                \STATE $\pi_i^\theta(t+1)\leftarrow (1-\varepsilon)\softmax(\beta Q_i(t+1))+\varepsilon\,\mathrm{Unif}(\Ai)$
                \STATE $\mathcal{J}\leftarrow \mathcal{J} - \log \pi_i^\theta(t+1)\bigl(a_i(t+1)\bigr)$
            \ENDFOR
        \ENDFOR
        \STATE $\mathcal{J}\leftarrow \mathcal{J} + \lambda \mathcal{R}(\theta)$
        \STATE Update $\theta$ by gradient descent on $\mathcal{J}$
    \ENDFOR
\ENDFOR
\RETURN $\hat\theta$
\end{algorithmic}
\end{algorithm}

\section{Theoretical Analysis}
This section provides two guarantees in a canonical setting: (i) identifiability of payoff \emph{differences} from conditional logit responses, and (ii) consistency of maximum likelihood estimation in an identifiable parametric class.
These results justify using action-only data for inverse mechanism learning and clarify the necessary gauge (additive constants).

\subsection{Identifiability under conditional logit response}
We first consider a simplified conditional response model that abstracts away internal learning state.
This model is the conditional building block of many stochastic learning dynamics, including the logit-Q dynamics in DIML.

\begin{assumption}[Conditional logit response]
\label{assump:logit}
Fix $\beta>0$.
For each time step $t$ and each agent $i$, conditional on opponent action $a_{-i}(t)$ the agent chooses action $a_i(t)$ according to
\begin{equation}
\label{eq:cond_logit}
\mathbb{P}_\theta(a_i(t)=a \mid a_{-i}(t)) = \frac{\exp(\beta\,u_i^\theta(a,a_{-i}(t)))}{\sum_{a'\in\Ai}\exp(\beta\,u_i^\theta(a',a_{-i}(t)))}.
\end{equation}
\end{assumption}

\begin{theorem}[Identifiability of payoff differences]
\label{thm:ident_diffs}
Fix agent $i$ and opponent action $a_{-i}\in\Ami$.
Under Assumption~\ref{assump:logit} with known $\beta>0$, the conditional choice probabilities $\{\mathbb{P}_\theta(a_i=a\mid a_{-i})\}_{a\in\Ai}$ identify the payoff differences
$
u_i^\theta(a,a_{-i}) - u_i^\theta(a',a_{-i})
$
for all $a,a'\in\Ai$.
Equivalently, the payoff vector $u_i^\theta(\cdot,a_{-i})$ is identifiable up to an additive constant $c_i(a_{-i})$ independent of $a$.
\end{theorem}

\begin{proof}
Let $p(a)\coloneqq \mathbb{P}_\theta(a_i=a\mid a_{-i})$.
For any $a,a'\in\Ai$,
\[
\frac{p(a)}{p(a')} = \exp(\beta(u_i^\theta(a,a_{-i})-u_i^\theta(a',a_{-i}))).
\]
Taking logs yields $u_i^\theta(a,a_{-i})-u_i^\theta(a',a_{-i}) = \frac{1}{\beta}\left(\log p(a)-\log p(a')\right),$
which is uniquely determined by $p(\cdot)$ and $\beta$.
Thus all differences are identified; adding a constant $c_i(a_{-i})$ to $u_i^\theta(\cdot,a_{-i})$ leaves~\eqref{eq:cond_logit} unchanged. 
\end{proof}

Theorem~\ref{thm:ident_diffs} clarifies an inherent gauge symmetry: only \emph{relative} incentives matter for behavior under logit response.

\begin{corollary}[Gauge fixing via normalization]
\label{cor:gauge}
If we impose, for all $i$ and all $a_{-i}\in\Ami$, the normalization
$
\sum_{a_i\in\Ai} u_i^\theta(a_i,a_{-i}) = 0,
$
then the conditional choice probabilities identify $u_i^\theta(a_i,a_{-i})$ uniquely for all $a_i,a_{-i}$.
\end{corollary}

\begin{proof}
By Theorem~\ref{thm:ident_diffs}, $u_i^\theta(\cdot,a_{-i})$ is determined up to additive constant $c_i(a_{-i})$.
The normalization fixes $c_i(a_{-i})$ uniquely since shifting by $c$ changes the sum by $|\Ai|c$.
\end{proof}

\subsection{Consistency of maximum likelihood estimation in a parametric mechanism class}
We state a standard M-estimation consistency result specialized to the conditional logit model, which applies directly to DIML when the learning dynamics induce conditionals of the form~\eqref{eq:cond_logit} and the model is correctly specified.

Let $\ell(\theta; a_i, a_{-i})$ denote the negative log-likelihood of observing $a_i$ given $a_{-i}$ under~\eqref{eq:cond_logit}.
Define the population risk $L(\theta)\coloneqq \E[\ell(\theta; a_i, a_{-i})]$ under the true data-generating parameter $\theta^\star$.

\begin{theorem}[Consistency of MLE (conditional logit)]
\label{thm:mle}
Assume:
(i) data $\{(a_i^{(k)},a_{-i}^{(k)})\}_{k=1}^N$ are i.i.d.\ from~\eqref{eq:cond_logit} induced by $\theta^\star$,
(ii) $\theta$ belongs to a compact parameter space $\Theta$,
(iii) $\ell(\theta;\cdot)$ is continuous in $\theta$ and dominated by an integrable envelope,
(iv) $L(\theta)$ has a unique minimizer at $\theta^\star$ (e.g., identifiability after gauge fixing in Corollary~\ref{cor:gauge}).
Let $\hat\theta_N\in\arg\min_{\theta\in\Theta}\frac{1}{N}\sum_{k=1}^N \ell(\theta; a_i^{(k)},a_{-i}^{(k)})$ be an MLE.
Then $\hat\theta_N \xrightarrow{p} \theta^\star$ as $N\to\infty$.
\end{theorem}

\begin{proof}
Under (ii)--(iii), the uniform law of large numbers implies $\sup_{\theta\in\Theta}\left|\frac{1}{N}\sum_{k=1}^N \ell(\theta;\cdot) - L(\theta)\right|\to 0$ in probability.
Under (iv), $L(\theta)$ has a unique minimizer at $\theta^\star$.
Standard M-estimation consistency (argmin continuity) yields $\hat\theta_N\xrightarrow{p}\theta^\star$.
\end{proof}

\begin{remark}[Unknown temperature]
If $\beta$ is unknown, Theorem~\ref{thm:ident_diffs} identifies payoff differences only up to a positive scale factor.
DIML can either estimate $\beta$ jointly (adding it to $\theta$) or fix $\beta$ and interpret inferred payoffs up to scale; regularization on payout magnitude further stabilizes learning.
\end{remark}

\section{Experiments}

\begin{figure*}[t]
    \centering

    \begin{subfigure}[t]{0.49\textwidth}
        \centering
        \includegraphics[width=\linewidth]{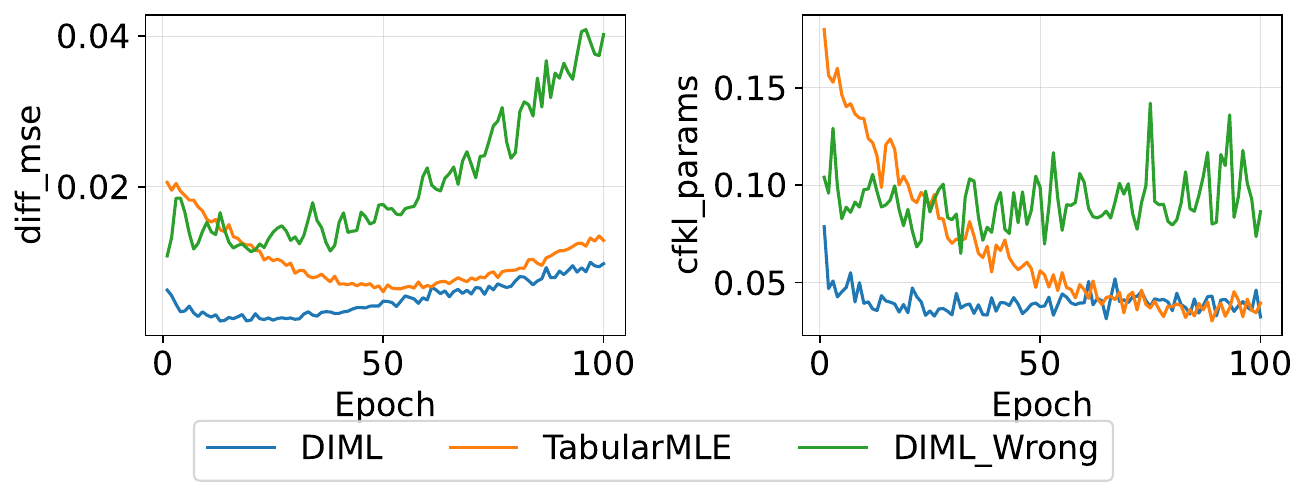}
        \caption{E1: Unstructured neural mechanism.}
        \label{fig:e1_grid}
    \end{subfigure}
    \begin{subfigure}[t]{0.49\textwidth}
        \centering
        \includegraphics[width=\linewidth]{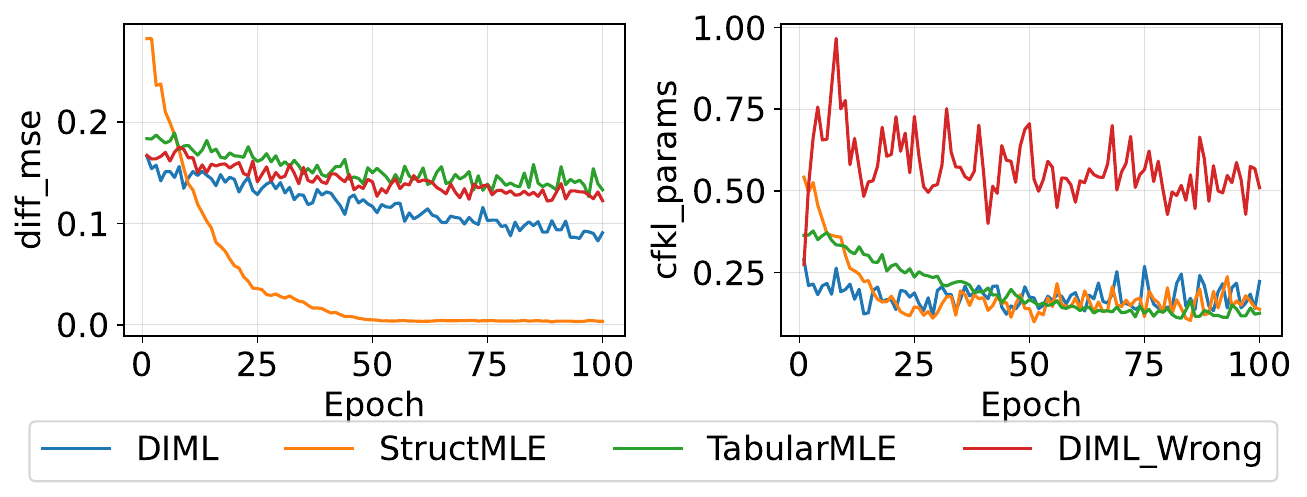}
        \caption{E2: Congestion tolling.}
        \label{fig:e2_grid}
    \end{subfigure}

    \vspace{0.6em}

    \begin{subfigure}[t]{0.49\textwidth}
        \centering
        \includegraphics[width=\linewidth]{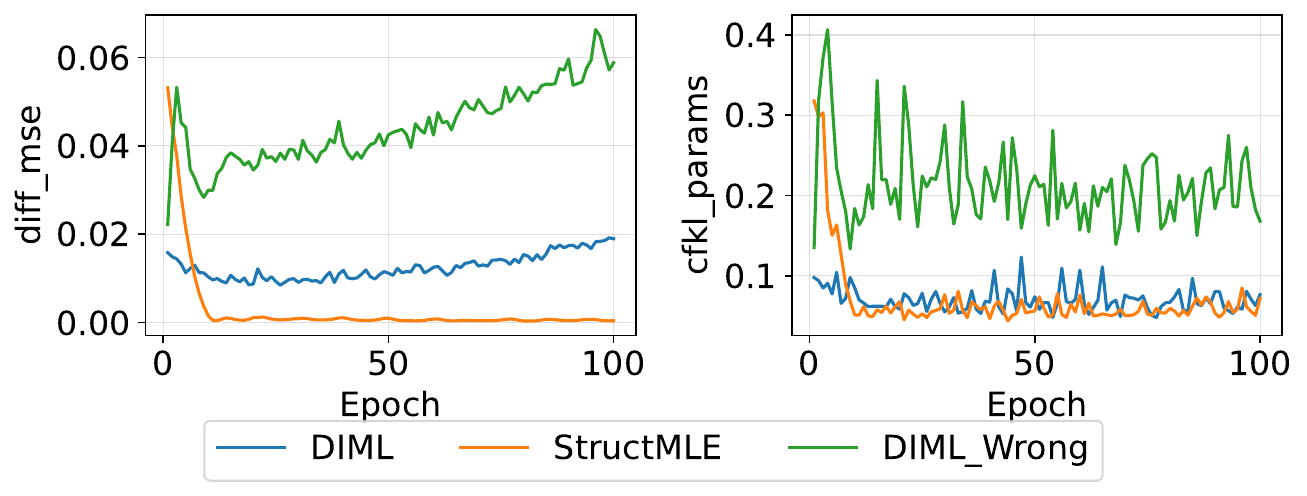}
        \caption{E3: Public goods.}
        \label{fig:e3_grid}
    \end{subfigure}
    \begin{subfigure}[t]{0.49\textwidth}
        \centering
        \includegraphics[width=\linewidth]{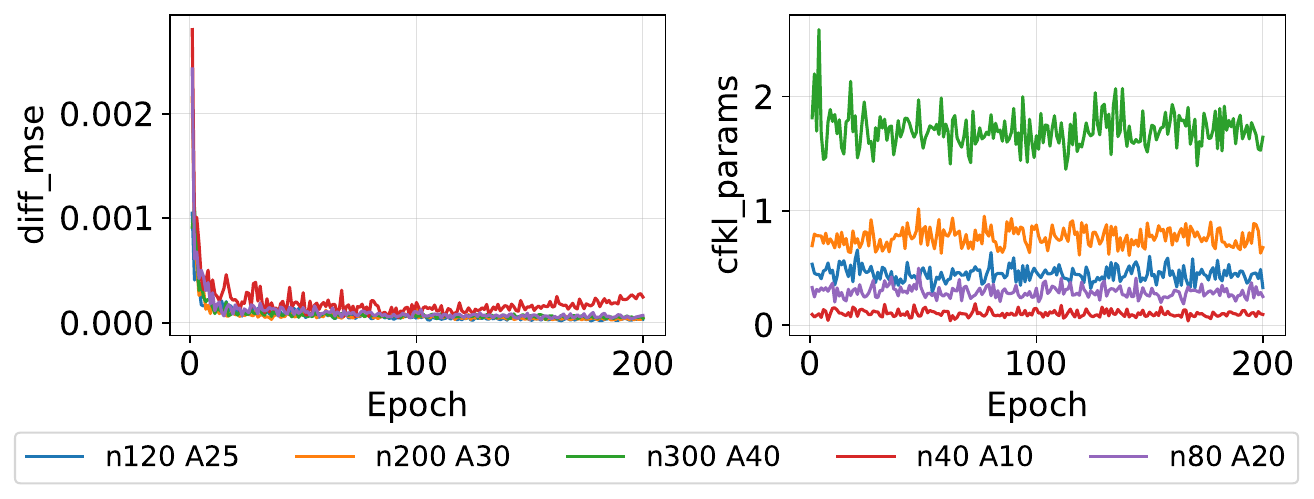}
        \caption{E4: Large-scale scaling comparison.}
        \label{fig:e7_grid}
    \end{subfigure}

    \caption{\textbf{Mechanism recovery and counterfactual validity.}
    Left panel in each subfigure: payoff-difference error (\texttt{diff\_mse}); right panel: counterfactual KL under learner-parameter shifts (\texttt{cfkl\_params}).}
    \label{fig:diml_grid}
\end{figure*}

We evaluate DIML across progressively more complex environments and regimes. The experiments are designed to answer three core research questions:

\emph{- RQ1 (On-support identifiability)}: Can the arbitrary underlying incentive mechanism be correctly inferred from observed learning behavior, on the behaviorally visited region of the joint action space?

\emph{- RQ2 (Robustness to inductive bias)}: How sensitive is inference to misspecification of the agents’ learning dynamics, and how does it compare to behavior-only baselines?

\emph{- RQ3 (Scalability)}: Does the approach scale to large multi-agent systems where tabular inverse methods are computationally infeasible?

To answer these questions, we design four experiments (E1–E4) spanning neural mechanisms, structured economic environments, and large-scale anonymous games.


\subsection{Experiment Setup}

We now introduce the experiment setups in detail.

\textbf{Agents.} Across all experiments, agents interact repeatedly in a stateless (or contextual) normal-form game. Each agent maintains Q-values over actions and updates them using a logit Q-learning rule parameterized by learning rate $\alpha$, inverse temperature $\beta$, and exploration rate $\epsilon$. The generator learner (used to produce data) and the inference learner (assumed by DIML) may differ, enabling controlled misspecification tests.

\textbf{Mechanisms and Experiments.} A mechanism maps joint actions (or action-count summaries) to per-agent payoffs. In the experiments, we cover three types of mechanisms: 1) Unstructured neural mechanisms, 2) Structured mechanisms, and 3) Symmetric count-based mechanisms. We evaluate DIML in these mechanisms in separate experiments:

\emph{- E1 (Neural mechanism, 4 agents, 5 actions):} The mechanism is a fully-connected, unconstrained, randomized, and frozen neural network over joint actions.

\emph{- E2 \& E3 (Structured mechanism, 3 agents, 7 actions):} The mechanisms are structured economic rules, where payoffs depend on (i) congestion: agents choosing popular actions incur higher costs; and (ii) collective contributions, agents contribute to a public good with diminishing returns.

\emph{- E4 (Large scale, 40-300 agents, 10-40 actions):} Agents interact in a symmetric, anonymous game where payoffs depend only on an agent’s action and the aggregate action-count vector of others. The mechanism is a fully-connected, unconstrained, randomized, and frozen neural network over one agent's action and the aggregated count vector of all agents' actions.

\subsection{Baselines}
\label{sec:baselines}

We compare DIML against baselines that isolate different sources of inductive bias and test whether recovering incentives requires explicit modeling of learning dynamics.


\paragraph{Tabular maximum likelihood estimation (Tabular MLE).}
When the joint action space $\A$ is small enough to enumerate, we parameterize a general tabular mechanism $M_\theta(a_{1:n}) = U_\theta[a_{1:n}] \in \R^n, U_\theta \in \R^{|\A|\times n},$
and fit $\theta$ by maximum likelihood under the same unrolled logit-Q learning dynamics as DIML.
That is, Tabular MLE minimizes $\mathcal{L}(\theta;D)
=
-\sum_{\tau\in D}\log p_\theta(\tau),$
where $p_\theta(\tau)$ is defined by Equations~\eqref{eq:q_update}--\eqref{eq:traj_ll_factor}.
This baseline is oracle-like in expressivity but computationally feasible only when $|\A|$ is small; it is omitted in large-scale experiments where $|\A|^n$ is astronomically large.

\paragraph{Structured parametric MLE (StructMLE).}
In environments with known economic structure (congestion tolling and public goods), we include a correctly specified parametric mechanism family $M_\theta(a_{1:n}) = f_\theta(a_{1:n}),$
where $f_\theta$ encodes the true structural form (e.g., congestion costs as a function of action counts, or public-good benefits as a function of total contribution).
Parameters $\theta$ are estimated by maximum likelihood using the same unrolled learning dynamics and objective as DIML: $\hat\theta
\in
\arg\min_\theta
-\sum_{\tau\in D}\log p_\theta(\tau).$
This baseline isolates the benefit of correct inductive bias and serves as a reference for how well DIML’s flexible neural parameterization can match a well-specified structural model.

\paragraph{DIML with misspecified learner (DIML-Wrong).}
To test robustness to learning-rule misspecification, we include an ablation where the inference procedure assumes incorrect learning hyperparameters (e.g., an incorrect inverse temperature $\beta$), while keeping the same mechanism class and likelihood objective.
Formally, DIML-Wrong optimizes $-\sum_{\tau\in D}\log p_{\theta,\tilde\Pi}(\tau),$
where $\tilde\Pi$ denotes a misspecified learning model.
This baseline separates failures due to incorrect incentive modeling from those due to imperfect assumptions about agent learning dynamics.

\subsection{Metrics and reporting}
\label{sec:metrics}

Inverse mechanism learning must be evaluated beyond behavioral imitation.
A method may fit observed actions well while recovering incorrect incentives.
Accordingly, we report metrics that align with identifiability, likelihood-based behavioral fit, and counterfactual validity.

\paragraph{Payoff-difference error (\texttt{diff\_mse}).}
Under logit-based response models, utilities are identifiable only up to additive, action-independent constants.
We therefore evaluate recovery using payoff differences:
\begin{multline*}
\texttt{diff\_mse}
=
\E_{(i,a_{-i})}
\E_{a,a'\sim\mathrm{Unif}(\Ai)}
\Big[
\big(
(u_i(a,a_{-i})\\
-u_i(a',a_{-i}))
-
(\hat u_i(a,a_{-i})-\hat u_i(a',a_{-i}))
\big)^2
\Big],
\end{multline*}
where $u_i(a,a_{-i})=[M_{\theta^\star}(a,a_{-i})]_i$ and $\hat u_i$ is defined analogously.
The expectation is approximated by sampling agent--time contexts from held-out trajectories and random action pairs.
This metric directly corresponds to the identifiability guarantee in Theorem~\ref{thm:ident_diffs}.


\paragraph{Counterfactual validity under learning-rule shifts (\texttt{cfkl\_params}).}
To assess whether inferred mechanisms support reliable counterfactual reasoning, we change agent learning parameters $(\alpha,\beta,\varepsilon)$ and simulate new trajectories under both the ground-truth mechanism $M^\star$ and the inferred mechanism $\hat M$.
Let $p_{M}^{(\alpha',\beta',\varepsilon')}$ denote the induced joint-action distribution under the intervened learner.
We report
\[
\texttt{cfkl\_params}
=
\KL\!\left(
p_{M^\star}^{(\alpha',\beta',\varepsilon')}
\;\middle\|\;
p_{\hat M}^{(\alpha',\beta',\varepsilon')}
\right).
\]
For small joint action spaces, the KL divergence is computed exactly over joint-action frequencies.
For large-scale symmetric environments, we compute KL over a count-based hash representation of action histograms, which preserves the anonymity structure of the game.

\paragraph{Reporting protocol.}
Tabular MLE is reported only when computationally feasible; for large-scale settings it is omitted by design.
All metrics are computed on held-out test data using fixed random seeds for reproducibility.

\subsection{Results}
\label{sec:results}

Figure~\ref{fig:diml_grid} shows learning curves for payoff-difference error (\texttt{diff\_mse}) and counterfactual validity (\texttt{cfkl\_params}).
Across all experiments, three consistent observations emerge.
First, learning-aware inverse recovers incentives and counterfactuals.
In the unstructured setting (E1), DIML achieves the lowest \texttt{diff\_mse} and \texttt{cfkl\_params} throughout training.
Tabular MLE improves initially but plateaus at higher error, while DIML-Wrong diverges, especially in counterfactual KL.
This shows that modeling learning dynamics provides identifying signal beyond function approximation alone.
Second, correct structure helps.
In structured environments (E2--E3), the correctly specified baseline converges to near-zero \texttt{diff\_mse}.
DIML remains competitive despite lacking the true parametric form and consistently outperforms DIML-Wrong on both metrics.
DIML-Wrong incurs large counterfactual KL in all cases, demonstrating that accurate counterfactual inference is highly sensitive to learning-rule misspecification.
Third, in large-scale anonymous games (E7), DIML drives \texttt{diff\_mse} to very small values across all scales, showing no qualitative degradation.
In contrast, \texttt{cfkl\_params} increases smoothly with scale, reflecting the growing complexity of induced joint-action distributions rather than a failure of incentive recovery.

These three observations lead us to the answers to our research questions.
\textbf{RQ1 (Identifiability).} Low \texttt{diff\_mse} across all settings that are comparable with oracle methods confirms that payoff differences are identifiable from learning trajectories and can be recovered by DIML in practice, hence the mechanism.
\textbf{RQ2 (Inductive bias).} Correct structure is optimal when available, but DIML provides a robust alternative without requiring mechanism knowledge.  
\textbf{RQ3 (Scalability).} DIML scales in incentive recovery, while counterfactual distribution matching becomes the limiting factor at large scale due to the curse of dimensionality.

Overall, We showed that DIML reliably recovers identifiable incentives and supports counterfactual reasoning across unstructured, structured, and large-scale environments, provided the learning dynamics are correctly modeled.

\section{Limitations}
DIML relies on a modeled family of learning dynamics; misspecification can degrade recovery, as captured by DIML-WrongLearner.
Moreover, identifiability can fail without sufficient exploration and coverage of opponent action profiles.
Finally, in real systems mechanisms may be nonstationary; extending DIML to time-varying mechanisms is an important future direction.

\appendix

\section{Implementation Details (Appendix)}
We include a single-script experiment suite that covers the environments and baselines described in Section~6.

\section*{Ethical Statement}
This work develops methods to infer incentive mechanisms from observed behavior.
Potential beneficial uses include auditing incentive systems for fairness and transparency and diagnosing unintended incentives.
Potential harmful uses include reverse-engineering proprietary mechanisms for manipulation or collusion.
We recommend that applications incorporate governance constraints, responsible disclosure, and safeguards against misuse, and that evaluations consider robustness to strategic adversaries.


\bibliographystyle{named}
\bibliography{ijcai26}

\end{document}